\documentclass[11pt]{article} 

\usepackage[utf8]{inputenc} 

\usepackage{geometry} 
\usepackage{graphicx} 

\usepackage{booktabs} 
\usepackage{array} 
\usepackage{paralist} 
\usepackage{verbatim} 
\usepackage{subfig} 
\usepackage{amssymb,amsmath,amsfonts,latexsym,amsthm}
\usepackage{bbm}
\usepackage{color}

\usepackage{fancyhdr} 
\usepackage{sectsty}
\allsectionsfont{\sffamily\mdseries\upshape} 

\usepackage[nottoc,notlof,notlot]{tocbibind} 
\usepackage[titles,subfigure]{tocloft} 

\newcommand{\R}{\mathbb{R}}

\newcommand{\E}{\mathbb{E}}

\newtheorem{theorem}{Theorem}

\newtheorem{lemma}[theorem]{Lemma}

\title{Bounding the Error From Reference Set Kernel Maximum Mean Discrepancy}
\author{Alexander Cloninger}
\date{} 

\begin{document}
\maketitle

\begin{abstract}
In this paper, we bound the error induced by using a weighted skeletonization of two data sets for computing a two sample test with kernel maximum mean discrepancy.  The error is quantified in terms of the speed in which heat diffuses from those points to the rest of the data, as well as how flat the weights on the reference points are, and gives a non-asymptotic, non-probabilistic bound.  The result ties into the problem of the eigenvector triple product, which appears in a number of important problems.  The error bound also suggests an optimization scheme for choosing the best set of reference points and weights.  The method is tested on a several two sample test examples.
\end{abstract}

\section{Introduction}

The purpose of this short note is to quantify the minimal loss of power created by subsampling the kernel used in two sample testing via kernel maximum mean discrepancy (MMD) \cite{cheng2017Anisotropic}.  This paper serves to demonstrate that the reduction in computational complexity of the asymmetric test \cite{cheng2017Anisotropic} does not significantly bias or increase the variance of the two sample statistic calculated using functions sampled from a reproducing kernel Hilbert space \cite{gretton2012kernel}.  The results rely on recent results \cite{linderman2018integration,steinerberger2017spherical} about optimal quadrature points and weights for integrating eigenfunctions of a Laplacian on some domain.

The standard form of MMD \cite{gretton2012kernel} seeks to measure the distance between two distributions $p$ and $q$ given $n$ points sampled from $p$ (called $X$) and $m$ points sampled from $q$ (caled $Y$).  This is done by constructing a kernel $K:(X\cup Y)\times (X\cup Y) \rightarrow [0,1]$ that measures similarity between data points.
The MDM statistic is found by comparing inter and intra class affinities between $X$ and $Y$ via
\begin{align*}
MMD(X,Y) = \frac{1}{N^2} \sum_{x\in X, x'\in X} K(x,x') +  \frac{1}{M^2} \sum_{y\in Y, y'\in Y} K(y,y')  - 2 \frac{1}{MN} \sum_{x\in X, y\in Y} K(x,y). 
\end{align*}
The largest issue in computation is building the $(n+m)\times (n+m)$ size kernel for comparing the point values.  This can be reduced by use of approximate nearest neighbors or the linear time MMD which only compares consecutive points in an arbitrary index.  However, these introduce an element of randomness and significantly higher variance in the test statistics.

In \cite{cheng2017Anisotropic}, we introduce two main notions for improving MMD: only comparing to a small, fixed reference set of points in the space, and building a local covariance matrix around these fixed points to deal with low dimensional data.  This paper only focuses on the former, and seeks to analyze the error in MMD introduced by subselecting the reference set to compare two distributions to.  

The speed of computing MMD may not be of large concern when only comparing two data sets $X$ and $Y$, but it becomes a factor when dealing with a large number of datasets $\{X_i\}_{i=1}^C$.  This is because computing $MMD(X_i, X_j)$ for all $i$ and $j$ requires $C^2$ kernels to be built.  Significant reduction can be made by only comparing to a fixed reference set, as this means we only have to build $C$ kernels and each of which is size $(n+m)\times ($number reference points$)$.

The main question introduced in this framework is how to choose and weight the referece points in order to best approximate the full $MMD$.  We domonstrated in \cite{cheng2017Anisotropic} that random sampling is a fairly robust method of choosing the reference points.  In this paper, we also introduce a notion of weighing the reference points, and demonstrate that this improves upon the empirical successes of equally weighted random sampling.  Also, we bound the error in MMD in terms of the location of the reference points and associated weights.  Qualatatively, the bound says that the error is determined by placing small ``heat sources'' of varying weights at the associated reference points, and measuring how flat the heat distribution is on the rest of the data after some small amout of diffusion time.  We quantify this statemen in Section \ref{thm}, and discuss the assumptions to the theorem in Section \ref{tripleprod}.    This theorem also suggests a better method of choosing the reference points and weights prior to building the test statistic.  We discuss how to better choose the reference points and improve upon the bounds in Section \ref{examples}.

\section{Diffusion Bounds for MMD}\label{thm}
%
We begin by formally introducing the notion of reference points and weighted reference point MMD.  Take two datasets $X\sim p$ and $Y\sim q$, where $p$ and $q$ are probability distributions in $\R^d$, with $|X|=n$ and $|Y|=m$.  We restrict ourselves to a class of kernels $K$ on the data that we refer to as \emph{random walk kernels}, which can be represented as
\begin{align*}
K(x,y) = \sum_{z\in X\cup Y} k(x,z) k(z,y), & \textnormal{ for some kernel } k.
\end{align*}
This is a large class of kernels which, intuitively, represents the affinity between two points as the probability of stepping from $x$ to any point in the data, and then stepping to $y$.  A simple examle is to take $k$ as the Gaussian $k(x,z) = e^{-\|x-y\|^2/\sigma^2}$, and this makes $K = k^2$.

The reference points are a smaller set of points $R\subset (X\cup Y)$ selected from the data.  The reference point kernel measures the similarity between all points and the reference set, $k_R:(X\cup Y)\times R \rightarrow \R^+$, which leads to a reference point MMD
\begin{align*}
MMD(X,Y;k_R)^2 = \frac{1}{|R|} \sum_{r\in R} \left(\frac{\sqrt{n+m}}{n} \sum_{x\in X} k_R(x,r) -\frac{\sqrt{n+m}}{m} \sum_{y\in Y} k_R(y,r)\right)^2.
\end{align*}
The renormalization of $\frac{\sqrt{n+m}}{n}$ comes from the desire to put $MMD(X,Y;k_R)$ on the same scale as $MMD(X,Y;K)$.

These reference points can also be assigned weights $a:R\rightarrow [0,1]$ that satisfy $\sum_{r\in R} a(r) = 1$.  This leads to the weighted reference point MMD
\begin{align*}
MMD_a(X,Y;k_R)^2 = \sum_{r\in R} a(r)  \left(\frac{\sqrt{n+m}}{n} \sum_{x\in X} k_R(x,r) - \frac{\sqrt{n+m}}{m} \sum_{y\in Y} k_R(y,r)\right)^2.
\end{align*}

Theorem \ref{thm:mainthm} bounds the error between $MMD_a(X,Y;k_R)$ and $MMD(X,Y;K)$, in terms of $R$ and $a$, by thinking of $MMD(X,Y;K)$ as being the norm of a particular function on the data and using results of \cite{linderman2018integration} for subsampling graphs.  However, our problem falls a little outside the purview of  \cite{linderman2018integration}, as our particular function has a high frequency component.  We account for this with an assumption below.  We are now prepared to introduce the following lemmas and theorem.

\begin{lemma}\label{lemma:mmd}
$MMD(X,Y;K)^2 = \frac{1}{n+m} \sum_{z\in X\cup Y} f(z)$ for 
\begin{align*}
f(z) = \left(\frac{\sqrt{n+m}}{n} \sum_{x\in X} k(x,z) - \frac{\sqrt{n+m}}{m} \sum_{y\in Y} k(y,z) \right)^2.
\end{align*}
\end{lemma}
\begin{proof}
This is a simple calculation, but we'll elaborate here for clarity.  
\begin{align*}
MMD^2(X,Y;K) &= \E_{x,x'\in X} [K(x,x')] + \E_{y,y'\in Y} [K(y,y')] - 2 \E_{x\in X,y\in Y} [K(x,y)] \\
&=  \E_{x,x'\in X} [\sum_z k(x,z) k(z,x')] + \E_{y,y'\in Y} [\sum_z k(y,z) k(z,y')] \\
&\hspace{.2in} - 2 \E_{x\in X,y\in Y} [\sum_z k(x,z) k(z,y)] \\
&= \sum_{z\in X\cup Y} \left( \frac{1}{n}\sum_{x\in X} k(x,z) - \frac{1}{m}\sum_{y\in Y} k(y,z) \right)^2.\\
&= \frac{1}{n+m}\sum_{z\in X\cup Y}  f(z),
\end{align*}
for $f(z) = \left(\frac{\sqrt{n+m}}{n} \sum_{x\in X} k(x,z) - \frac{\sqrt{n+m}}{m} \sum_{y\in Y} k(y,z) \right)^2$
\end{proof}

In order to apply any results from \cite{linderman2018integration}, we must consider not only the kernel $K$ but also the renormalized \emph{lazy random walk} kernel.  We will first construct the kernel, and then discuss why this isn't a hinderance for our methods.  Let $D_x  =  \sum_{z \in X\cup Y} K(x,z)$, and let $d_{max}= \max_{x\in X\cup Y} D_x$.   The lazy walk kernel can then be written as $P = \frac{1}{d_{max}}\left(K - D + d_{\max}\cdot I \right)$, where $D$ is the diagonal matrix of $D_x$.  The only difference between $K$ and $P$ (other than a global scaling by $d_{max}$) is the altered diagonal of the matrix.  We note that, even if we only know $k_R$ across a random set of reference points, it's still possible to estimate $D$ by appropriately rescaling $\sum_{r\in R} k(z,r)$ according to $(n+m)/|R|$. 

Now we are prepared to state the main assumption that must be made in order to prove the main result.  We take $\Phi_\lambda$ to be the subspace of all eigenfunctions of $P$ who's eigenvalue satisfies $|\lambda_\ell| > \lambda$.  The assumption is that $\|(I-\Phi_\lambda \Phi_\lambda^*)f\|$ is small.  In other words, we must assume that $f$ projects only a small amount of energy onto the high frequency eigenfunctions of $K$.  This assumption arises from the fact that the result boils down to comparing eigenfunctions of $P$ and $K=\Psi \Sigma \Psi^*$, which are highly related to one another.  In particular, the result requires examining the pointwise product of eigenfunctions $\langle \phi_i, \psi_j \psi_k\rangle$.  Even in the case that $\Phi = \Psi$, this becomes the triple product coefficients, for which there are no bounds in general.  This assumption is discussed in simple cases and with empirical evidence in Section \ref{tripleprod}.  We are now prepared to address the main theorem.

\begin{theorem}\label{thm:mainthm}
	Let $MMD(X,Y;K) = \tau$, assume $\|\Phi_\lambda \Phi_\lambda^*f\| = (1-\epsilon) \tau$.  Then for a given reference set $R$ and weights $a$, 
	\begin{align*}
	|MMD_a^2 & (X,Y;k_R) - MMD^2(X,Y;K)| < \\
	&\tau \left[\frac{(1-\epsilon)}{\lambda} \left(\frac{1}{d_{max}^2} \left\| (K + d_{max}\cdot I - D) \sum_w a_w \delta_w \right\|_2^2 - \frac{1}{n+m}\right)^{1/2}  + \left( \frac{1}{\sqrt{n+m}} + \|a_w\| \right) \epsilon \right].
	\end{align*}

\end{theorem}

\begin{proof}

Using Lemma \ref{lemma:mmd}, we know that $$\left|MMD_a(X,Y;k_R) - MMD(X,Y;K)\right| = \left|\sum_{r\in R} a(r) f(r) - \frac{1}{n+m} \sum_{z\in X\cup Y} f(z)\right|.$$  Now we examine the approximation error
\begin{align*}
\left| \frac{1}{n+m} \sum_v f(v) - \sum_w a_w f(w) \right| \le& \bigg| \frac{1}{n+m} \sum_v (\Phi_\lambda \Phi_\lambda^* f)(v) - \sum_w a_w (\Phi_\lambda \Phi_\lambda^*  f) (w) \bigg| +\\
& \bigg| \frac{1}{n+m} \sum_v ((I - \Phi_\lambda \Phi_\lambda^*) f)(v) \bigg| + \bigg| \sum_w a_w ((I - \Phi_\lambda \Phi_\lambda^*) f)(w) \bigg|\\
\le& \frac{\|f_\lambda\|}{\lambda} \left(\left\| P \sum_w a_w \delta_w \right\|_2^2 - \frac{1}{n+m}\right)^{1/2}  + \\
& \frac{1}{\sqrt{n+m}} \|(I-\Phi_\lambda \Phi_\lambda^*) f\| + \|a_w\| \cdot \|(I-\Phi_\lambda \Phi_\lambda^*) f\| \\
\le& \frac{(1-\epsilon)\tau}{\lambda} \left(\left\| P \sum_w a_w \delta_w \right\|_2^2 - \frac{1}{n+m}\right)^{1/2}  + \left( \frac{1}{\sqrt{n+m}} + \|a_w\| \right) \epsilon \tau.
\end{align*}
This is $\ell=2$ in \cite{linderman2018integration}, using the fact that $k$ is the squareroot of $K$.  The bounding of the residual terms comes from a Cauchy-Schwartz inequality.

\end{proof}

We note that both terms $ \left(\left\| P \sum_w a_w \delta_w \right\|_2^2 - \frac{1}{n+m}\right)^{1/2}$ and $\|a\|$ decay like $\frac{1}{\sqrt{|R|}}$.  This implies that, for a fixed $\lambda$ and $\epsilon$, the errror converges to the full $MMD(X,Y;K)$ as $|R|$ grows.  Also, this provides guarantees for the randomly choosen, equally weighted reference points from \cite{cheng2017Anisotropic}, in which $a_R(r) = \frac{1}{|R|}$ and the random selection of points can be tested for diffusion flattness on the given data sets.

\section{Concentration of Pointwise Product Energy}\label{tripleprod}

Now we must argue for our assumption that the projection $\Phi_\lambda \Phi_\lambda^*$ conserves most of the energy of $f$.  We use the notation that $P = \Phi \Lambda \Phi^*$, $K = \Psi \Sigma \Psi^*$, and because $K=k^2$ we know that $k = \Psi \Sigma^{1/2} \Psi^*$.  Now we note that
\footnotesize
\begin{align*}
\Phi_{<\lambda} \Phi_{<\lambda}^* f =& \sum_{\{\ell:|\Lambda_\ell| < \lambda\}} \sum_{z\in X\cup Y} \Phi_{\cdot, \ell} \Phi_{z,\ell} \left[ \left(\frac{\sqrt{n+m}}{n} \sum_x \Psi(x) - \frac{\sqrt{n+m}}{m}\sum_y \Psi(y)\right) \Sigma^{1/2} \Psi(z)^* \right]^2 \\
=& \sum_{\{\ell:|\Lambda_\ell| < \lambda\}} \Phi_{\cdot, \ell} \sum_z  \Phi_{z, \ell}  \Bigg( \sum_k \sum_{k'} \Phi_{z,k}\Phi_{z,k'} \Sigma^{1/2}_{k,k} \Sigma^{1/2}_{k',k'} \cdot\\
&\left(\frac{\sqrt{n+m}}{n}\sum_x \Psi_{x,k} - \frac{\sqrt{n+m}}{m}\sum_y \Psi_{y,k}\right) \left(\frac{\sqrt{n+m}}{n}\sum_x \Psi_{x,k'} - \frac{\sqrt{n+m}}{m}\sum_y \Psi_{y,k'}\right) \Bigg) \\
=& \sum_{\{\ell:|\Lambda_\ell| < \lambda\}} \Phi_{\cdot, \ell}\left\langle \Phi_{\cdot, \ell},  \sum_{k,k'} c_k c_{k'} \Psi_{\cdot, k} \Psi_{\cdot,k'} \right\rangle,
\end{align*}
\normalsize
where $c_k = \Sigma^{1/2}_{k,k} \left( \frac{\sqrt{n+m}}{n}\sum_x \Psi_{x,k} - \frac{\sqrt{n+m}}{m}\sum_y \Psi_{y,k}\right)$ and $\Psi_{\cdot,k} \Psi_{\cdot,k'}$ is understood to be the pointwise product fo the vectors.

If we assume $K\in L^4(\mathcal{X}\times \mathcal{X})$, we know that $c_k \in \ell^2$.  However, this does not preclude the pointwise product of low frequency eigenfunctions (i.e. eigenfunctions with significantly large eigenvalues) from containing significant energy in the high frequency regime.  Or similiarly, it doesn't prevent a high frequency and low frequency product from having substantial high frequency energy.  

The question of how the energy of pointwise products of eigenfunctions is spread over the spectrum is an open question.  Even in the case where $D_{i,i} = d$, which means the degree of each node is constant and makes $\Phi = \Psi$, the problem of pointwise product energy has only been partially studied \cite{steinerberger2017spectral,cloninger2018dual,filbir2011marcinkiewicz,sarnak1994integrals}.  While the problem may remain open in general, we have an additional term that regulates the size of the high frequency component, namely the empirical averages of the eigenfunctions across $X$ and $Y$.  This term additionally penalizes high frequency eigenvectors, and along with the small eigenvalues, regulates the energy of the high-high products.  

For example, \cite{burq2004multilinear} shows for compact manifolds that 
$$\|\Psi_k  \Psi_{k'}\|_2 \le -C \log(\max(\sigma_k,\sigma_{k'}))^{\frac{d-2}{2}} \textnormal{ for } d\ge 4.$$
Note, for $d=2,3$ there are similar bounds of $ \log(\max(\sigma_k,\sigma_{k'}))^{\frac{1}{4}}$ and $ \log(\max(\sigma_k,\sigma_{k'}))^{\frac{1}{2}+\epsilon}$, respectively. This would give us in those situations

\footnotesize
\begin{align*}
\Bigg\|\sum_{\{\ell:|\Lambda_\ell| < \lambda\}} \Phi_{\cdot, \ell}\left\langle \Phi_{\cdot, \ell},  \sum_{k,k'} c_k c_{k'} \Psi_{\cdot, k} \Psi_{\cdot,k'} \right\rangle\Bigg\|
\le& \Bigg\|\sum_{k<C,k'<C} c_k c_{k'}\sum_{\ell>L} \Phi_{\cdot, \ell}\left\langle \Phi_{\cdot, \ell}, \Psi_{\cdot, k} \Psi_{\cdot,k'} \right\rangle\Bigg\| + \\
&  2 \Bigg\| \sum_{k<C,k'>C} c_k c_{k'}\sum_{\ell>L} \Phi_{\cdot, \ell}\left\langle \Phi_{\cdot, \ell}, \Psi_{\cdot, k} \Psi_{\cdot,k'} \right\rangle\Bigg\| +\\
& \Bigg\|\sum_{k>C,k'>C} c_k c_{k'}\sum_{\ell>L} \Phi_{\cdot, \ell}\left\langle \Phi_{\cdot, \ell}, \Psi_{\cdot, k} \Psi_{\cdot,k'} \right\rangle\Bigg\| \\
\le& \sum_{k<C,k'<C} c_k c_{k'} \Bigg\|  \sum_{\ell>L}\Phi_{\cdot, \ell}\left\langle \Phi_{\cdot, \ell}, \Psi_{\cdot, k} \Psi_{\cdot,k'} \right\rangle \Bigg\| + \\
&  2\sum_{k<C,k'>C} c_k c_{k'} \Bigg\| \sum_{\ell>L}\Phi_{\cdot, \ell}\left\langle \Phi_{\cdot, \ell}, \Psi_{\cdot, k} \Psi_{\cdot,k'} \right\rangle\Bigg\| + \\
&\sum_{k>C,k'>C} \textnormal{Diff}_k \textnormal{Diff}_{k'} \sigma_k^{1/2}\sigma_{k'}^{1/2}\left( -C \log(\max(\sigma_k,\sigma_{k'}))^{\frac{d-2}{2}} \right)
\end{align*}
\normalsize
where $\textnormal{Diff}_k = \left( \frac{\sqrt{n+m}}{n}\sum_x \Phi_{x,k} - \frac{\sqrt{n+m}}{m}\sum_y \Phi_{y,k}\right)$.  This implies that the third term (the $k>C,k'>C$ term) can be bounded by use of the fact that $\sigma_k^{1/2} \log(\sigma_k)^{\frac{d-2}{2}}$ still decays quickly.  Similarly, the second term can also be bounded using the fact that $\sigma_k^{1/2}$ decays quickly and $C$ is chosen to be small.  

The main conjecture comes in bounding the first term.  Intuitively, the pointwise product of two low frequency eigenfunctions will remain somewhat low frequency, which means that the projection onto high frequency eigenvectors will be minimal.  While we cannot prove this statement in generality, we refer the reader to \cite{cloninger2018dual} for a large empirical study of Hadamard products of eigenfunctions.  Specifically, we note that low frequency eigenfunctions have a large $\|e^{t\Delta} \Psi_k \Psi_{k'}\|/\|\Psi_k \Psi_{k'}\|$ for small $k,k'$.  This implies that the vast majority of the enrgy is projected onto the low frequency eigenfunctions of $e^{t\Delta}$, meaning that $\|(
(I - \Phi_{\lambda} \Phi_{\lambda} ^* )\left(\Psi_k \Psi_{k'}\right)\|$ is small.

\begin{figure}[!h]
\footnotesize
\begin{tabular}{cc}
\includegraphics[width=.4\textwidth]{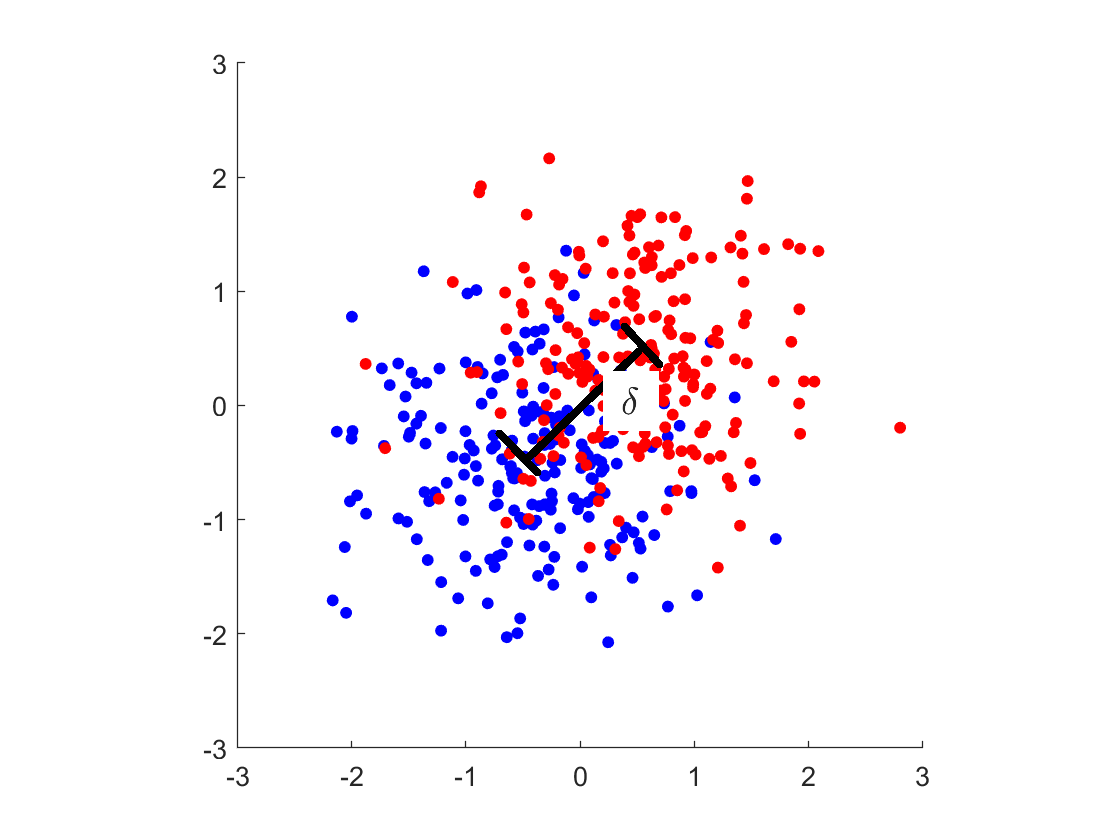} &
\includegraphics[width=.4\textwidth]{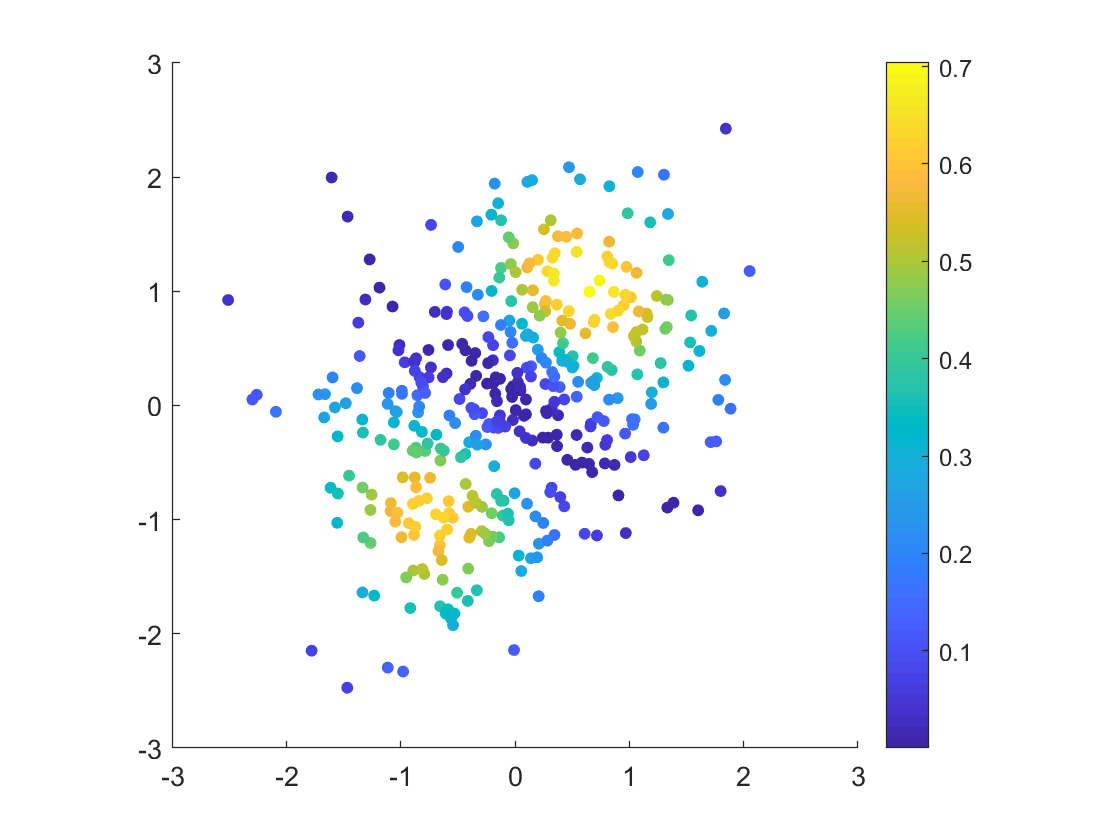} \\
Data Sep. by $\delta$ & Data Colored by $(\widehat{\mu_X} - \widehat{\mu_Y})^2$\\
\includegraphics[width=.4\textwidth]{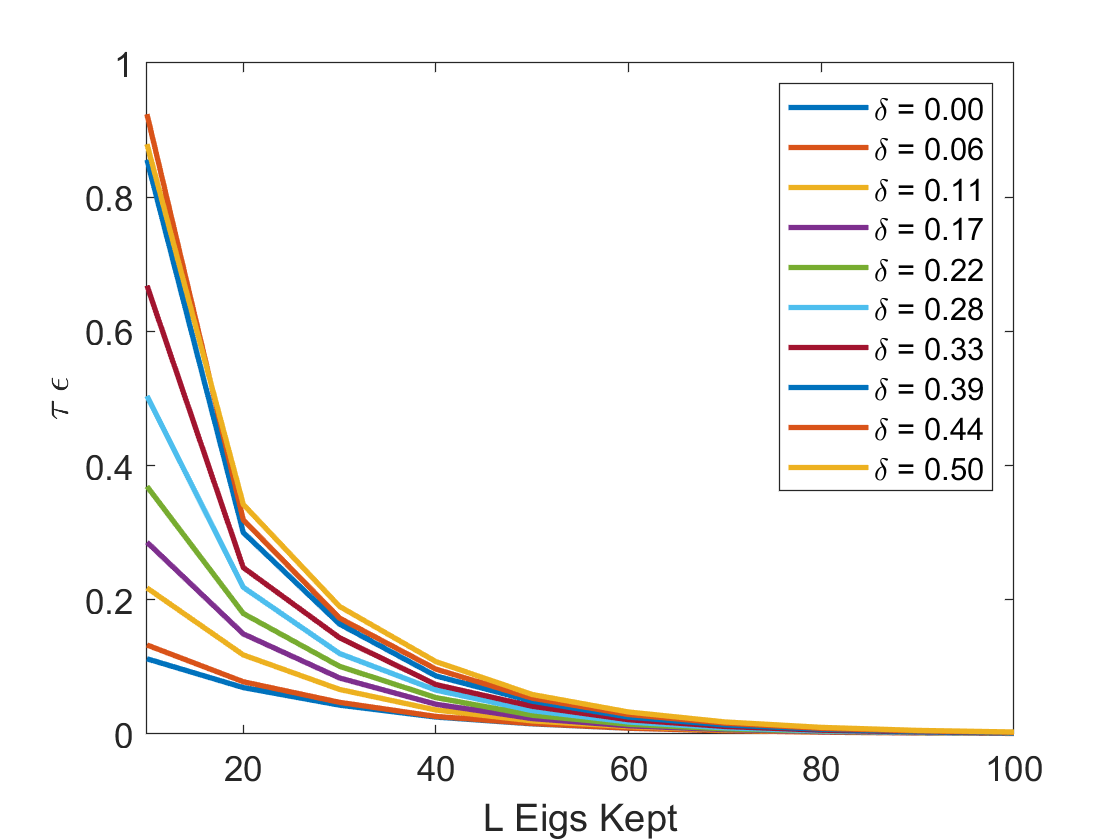} &
\includegraphics[width=.4\textwidth]{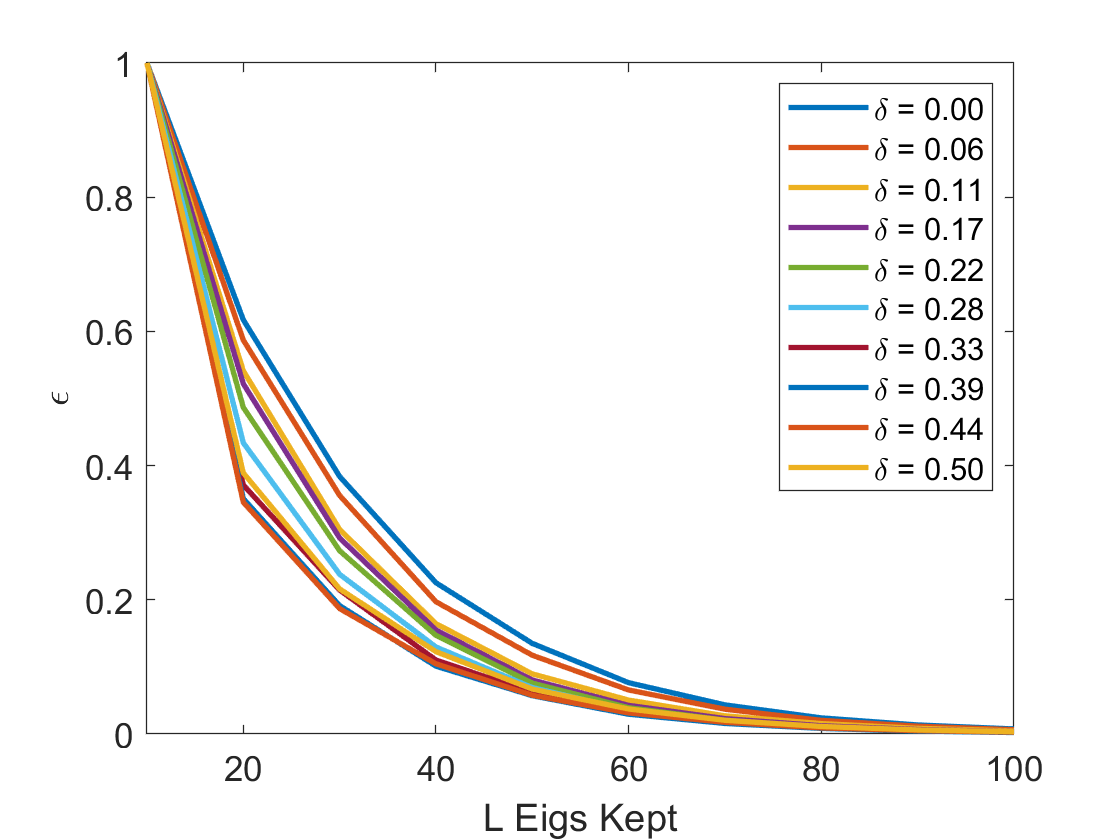} \\
Plot $\tau\epsilon$ varying $\delta$ and $\Phi_{<\lambda}$ & Plot $\epsilon$ varying $\delta$ and $\Phi_{<\lambda}$ 
\end{tabular}
\caption{Energy Remaining Averaged over 100 trials}\label{fig:gaussian2D}
\end{figure}

We empirically validate these statements in Figure \ref{fig:gaussian2D}.  We examine varying size shift of 2D gaussians consisting of $200$ points for each gaussian, with a kernel of bandwidth 0.5.  For each shift size, we run 100 trials and collect the size of $\|\Phi_{< \lambda_L} \Phi_{< \lambda_L}^* (\widehat{\mu_X} - \widehat{\mu_Y})^2\|$ for varying $\lambda_L$.  For all sizes of shifts, $\tau\epsilon$ decays quickly, and implies that it is only necessary to keep around $40$ eigenfunctions to bound most of the energy.  We also note that $\epsilon$ decays slowest when there is no shift (i.e. the null hypothesis), however in this situation $\tau$ is so much smaller that $\tau\epsilon$ remains small.


\section{Examples and Optimization}\label{examples}
The upper bound on $MMD_a(X,Y;k_R)$ suggests a possible algorithm for choosing the reference points, and associated weights, in a non-biased fashion for determining the two sample statistic.  Given a selection of reference points $R$, one can establish weights $a_w$ via
\begin{align}\label{eq:optWeight}
\arg\min_a  \bigg[\frac{(1-\epsilon)}{\lambda} & \left(\frac{1}{d_{max}^2} \left\| (K + d_{max}\cdot I - D) \sum_w a_w \delta_w \right\|_2^2 - \frac{1}{n+m}\right)^{1/2}  + \\
& \left( \frac{1}{\sqrt{n+m}} + \|a_w\| \right) \epsilon \bigg]\\
\textnormal{such that } & \|a_w\|_1 = 1,
\end{align}
for a selection of parameter $\lambda$ (which in tern selects $\epsilon$).  This enforces that the weights are optimized to diffuse across the entire dataset as quickly as possible, while still being sufficiently spread out in energy.  

Selection of the reference points themselves, $R$, is a more complicated problem.  It is possible to frame this problem as one of submodular optimization and create a convex relaxation \cite{hassani2017gradient}.  However, this would require knowledge of all columns of $K$ to frame the optimization, which defeats the purpose of using a reference set.  One heuristic that can be used is the following:
\begin{enumerate}
\item Randomly select $R\subset X\cup Y$ with $|R|\ll N$
\item Compute $k_R:{X\cup Y}\times R \rightarrow [0,1]$
\item If there exists $x \not\in R$ such that $\sum_{r\in R} k(x,r) < \delta$, set $R = R\cup \{x\}$
\item Estimte $D = \frac{n+m}{|R|}\sum_{r\in R} k_R(x,r)$
\end{enumerate}
This ensures that no points are too far from a reference point, while still being memory efficient.  

The benefit of optimizing in this way is that there is no a priori information brought in about differentiating $X$ from $Y$.  This prevents biasing of $MMD_a(X,Y;k_R)$ from the true $MMD(X,Y;K)$.  Also, computation of $MMD_a$ only requires computation of $O(N\cdot|R|\cdot d)$, as opposed to $O(N^2d)$ naively.

We demonstrate the use of this optimization and bound in a number of examples.  We note, the examples will be demonstrated using a non-traditional metric of comparison, namely a permutation test, rather than simply a plot of the MMD for various distributions.  This is because, while the reference point equal weighting $MMD(X,Y;k_R)$ is not as close to the true empricial $MMD(X,Y;K)$ as the weighted reference point $MMD_a(X,Y;k_R)$, $MMD(X,Y;k)$ is still equivalent to $MMD(X,Y;K)$ in expected value.  This means that displaying means over some number of trials won't reflect the importance of using the weights $a$.  Instead, we will focus on a comparison between $MMD(X,Y;K)$ and a number of permutations $MMD(Z_1,Z_2;K)$ where $Z_1,Z_2\subset X\cup Y$, and record the number of times $MMD(X,Y;K)$ is greater $(1-\alpha)$ of the permutations.  This yields a more reliable measure of accuracy (similar to measuring $|MMD(X,Y;K) - MMD(X,Y;k_R)|^2$), and is the ultimate objective of most two sample test algorithms.  This is equivalent to comparing the various statistics directly, as Theorem \ref{thm:mainthm} guarantees similarity to both $MMD(X,Y;K)$ and $MMD(Z_1, Z_2;K)$.

\subsection{Gaussian with small anomaly}
Let $p\sim N(0,I_5)$ and 
$$q\sim \begin{cases} N(0,I_5) & \textnormal{sampled with probability } (1-\delta)\\
N((2,2,2,2,2),0.1\cdot I_5) & \textnormal{sampled with probability } \delta
\end{cases},$$ 
for various values of $\delta$.  This effectively models two matching distributions, but with one developing a small anomalous cluster away from the bulk.  We display the average detected deviation for this example, for varying $\delta$ in Figure \ref{fig:2DGaussian}, using all points in $MMD(X,Y)$, $|R|=25$ randomly chosen reference points in $MMD_R(X,Y;k)$, and reference points and weights that are opitimized by \eqref{eq:optWeight} in $MMD_A(X,Y;k,a)$.  The optimized weights better track the full $MMD(X,Y)$ for varying $\delta$ than the randomized weights.  We note that $200$ points for each distribution and only $25$ total reference points is a stringent test of reeference point $MMD$, which we have purposely done for the purposes of highlighting the weighting.  If $|R|>50$ or so, then the simple randomized $MMD$ would match closely as well.

\begin{figure}[!h]
\centering
\includegraphics[width=.4\textwidth]{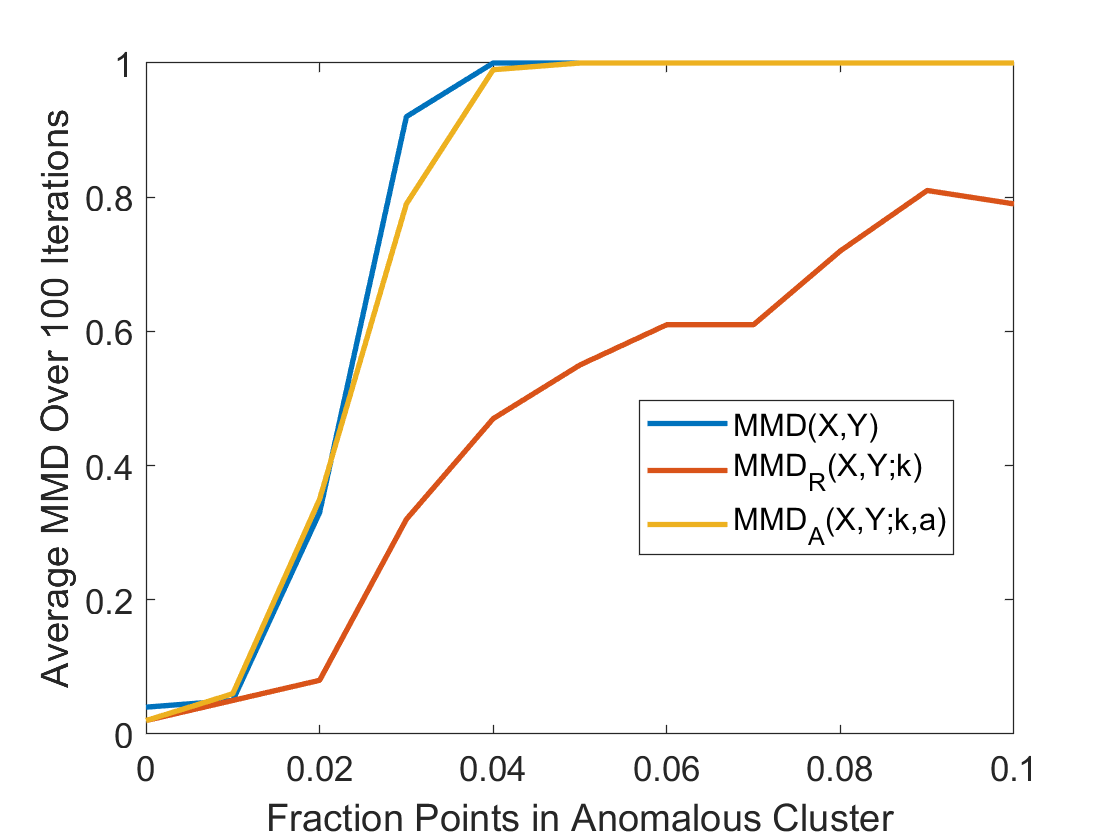} 
\caption{Full and weighted reference point MMD for a 5D Gaussian with an additional small anomalous cluster.   All curves compare the alternative against a permutation test and compute the probability of rejecting the null hypothesis.}\label{fig:2DGaussian}
\end{figure}

\subsection{Manifold local shift}
We let $p$ be a uniform distribution over a $5$ dimensional sphere, and $q$ be the same spherical distribution but with a slight elliptical increase of height delta forming on one direction.  
This models deviation of a lower-dimensional manifold embedded in a higher dimensional space, and a slight deviation that has significant effect on the eigenfunctions of the manifold Laplacian.  We display the average detected deviation in Figure \ref{fig:manifoldShift}.

\begin{figure}[!h]
\centering
\includegraphics[width=.4\textwidth]{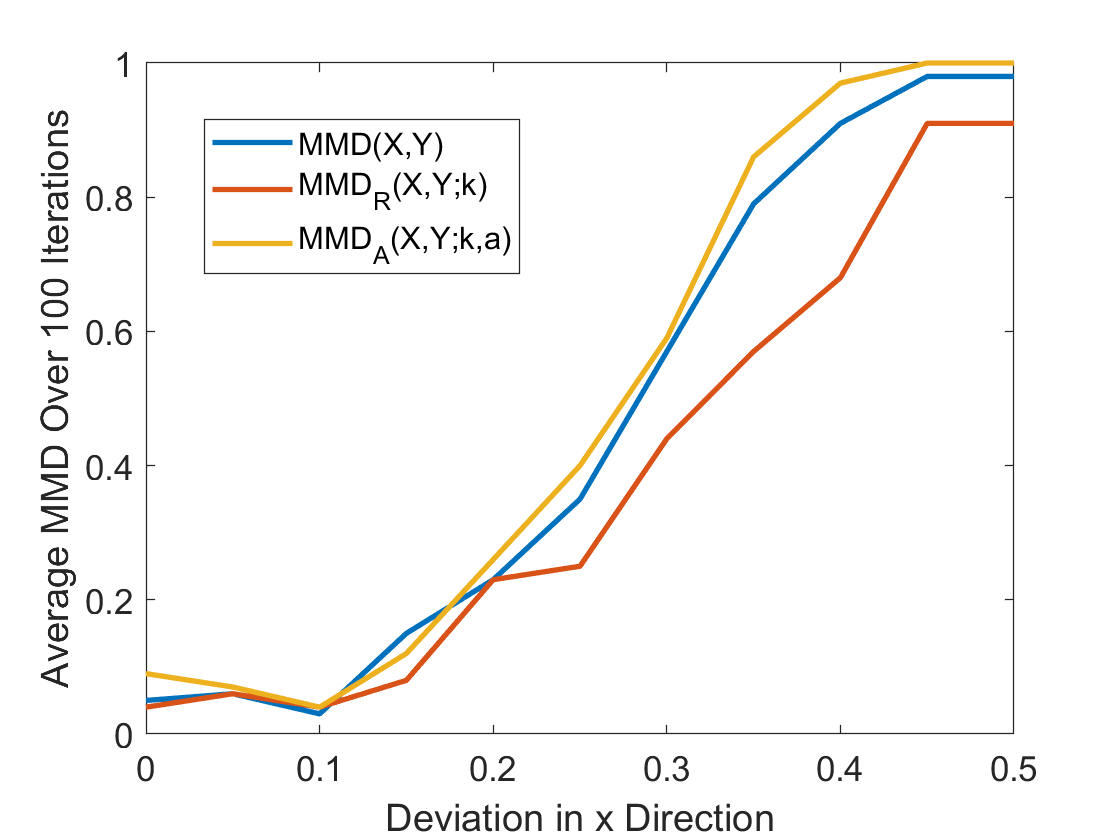} 
\caption{Full and weighted reference point MMD for differentiating a 5D sphere from a 5D elipse with principle axis length 
$[ 1+\delta , 1 , 1 , 1 , 1]$.   
All curves compare the alternative against a permutation test and compute the probability of rejecting the null hypothesis.}\label{fig:manifoldShift}
\end{figure}

\subsection{Gaussian mixture in 3D from paper}
Let $p$ be a gaussian mixture distribution with three highly anisotropic point clouds in $3D$, and $q$ to be a mean shift of each of these three clouds.   We examined this particular gaussian distribution example in \cite{cheng2017Anisotropic}, and showed that considering MMD with kernels that have anisotropic bandwidth has significantly more power than kernels with isotropic bandwidth.  In this paper, the goal is to show that weighting the small number of reference points yeilds an even more powerful test statistic.  Figure \ref{fig:gaussianmixture} shows an example of the clouds sampled from $p$ and $q$, as well as the relative power from the permutation test for isotropic kernels, anisotropic kernels at reference points with equal weight, and anisotropic kernels with weight vector $a$.  

This result will be different in flavor than the previous two.  In this situation, we are not attempting to match the traditional, full isotropic MMD.  By adding anisotropic kernels, we have already improved significantly on the power of the test, and thus have no desire to match the isotroipc MMD.  In this section, we wish to demonstrate that even here we can improve upon the power by using a non-uniform weighting on the reference points, and the weights can still be chosen by \eqref{eq:optWeight}.

\begin{figure}[!h]
\centering
\includegraphics[width=.4\textwidth]{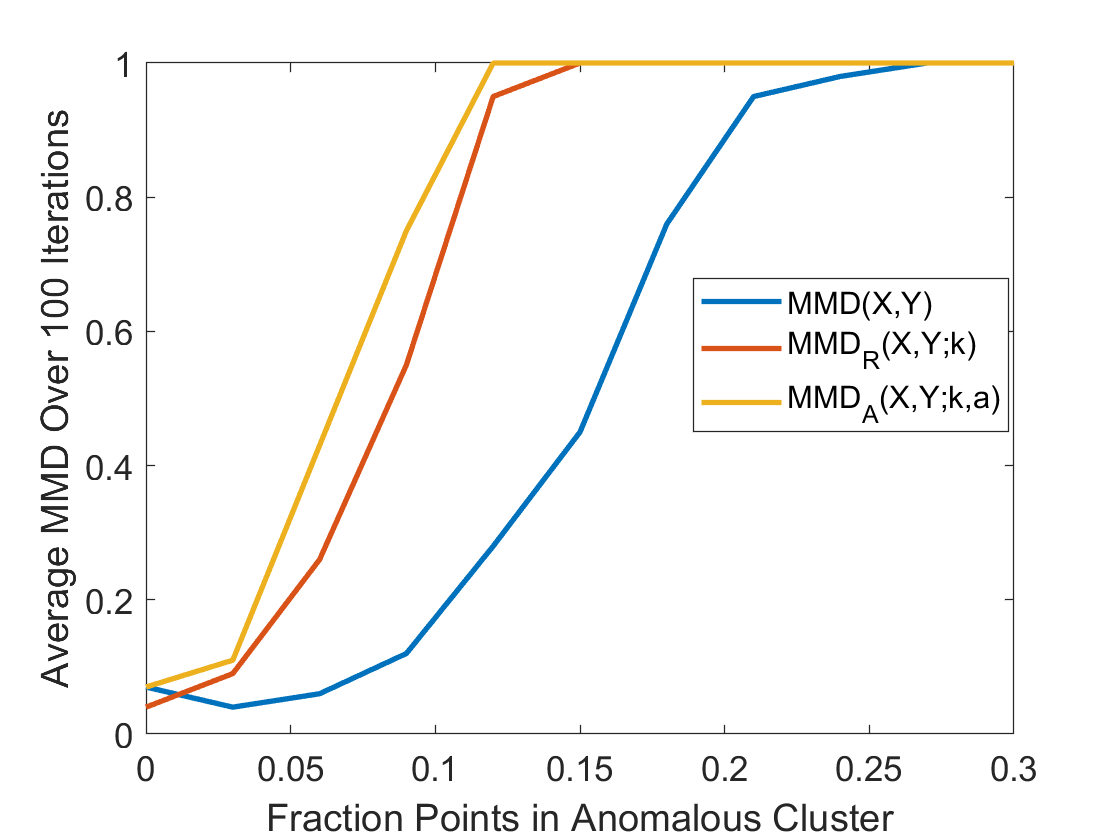} 
\caption{Comparing isotropic Gaussian MMD, reference point anisotropic MMD with randomly chosen equally weighted reference point, and reference point anisotroipc MMD with weighted equidistribued reference points.  All curves compare the alternative against a permutation test and compute the probability of rejecting the null hypothesis.}\label{fig:gaussianmixture}
\end{figure}

\section{Conclusions}
In this paper, we demonstrated that weighted reference point MMD can closely match the full MMD with significantly less computation and storage.  We bound the error in terms of heat diffusion from the reference points, and the bound suggests an optimization scheme for improving the reference point MMD even further.  There are several directions of future work spawning from this result, namely examining the estimate of $D$ and $d_{max}$ without building $K$, and fast methods for completing the optimization scheme through submodular optimization \cite{hassani2017gradient}.
%
%
%
%
%
%
%

\bibliography{Bibliography}{}
\bibliographystyle{plain}

\end{document}